\pdfoutput=1
\documentclass[11pt]{article}

\usepackage[letterpaper,margin=1in]{geometry}
\usepackage[T1]{fontenc}
\usepackage[utf8]{inputenc}
\usepackage{lmodern}
\usepackage{microtype}
\usepackage{booktabs}
\usepackage{amsmath,amssymb,amsthm}
\usepackage{array}
\usepackage{graphicx}
\usepackage{url}
\usepackage{xcolor}
\usepackage[numbers,sort&compress]{natbib}
\usepackage{tikz}
\usepackage{pgfplots}
\pgfplotsset{compat=1.18}
\usepackage[hidelinks]{hyperref}

\title{Validation-Frontier Representation Selection under Constrained Observation}
\author{Wesley Szu-Way Shu\\
Independent Researcher\\
\texttt{shuwesley@gmail.com}}
\date{}

\hypersetup{
  pdftitle={Validation-Frontier Representation Selection under Constrained Observation},
  pdfauthor={Wesley Szu-Way Shu},
  pdfsubject={arXiv preprint},
  pdfkeywords={artificial intelligence, representation selection, robust AI systems, constrained observation, dataset shift, validation frontier}
}

\newtheorem{definition}{Definition}

\newtheorem{proposition}{Proposition}

\begin{document}
\maketitle

\begin{abstract}
AI systems deployed outside clean benchmark settings often rely on observations that are incomplete, unstable, costly, or degraded by monitoring failures. This paper studies representation selection under constrained observation: choosing a state representation when raw accuracy is not the only operational criterion. We propose a validation-frontier selector that combines balanced accuracy with penalties for feature cost, overfit gap, and validation-test instability. In a focused public-tabular benchmark using three scikit-learn datasets, five observation regimes, 45 matched task cells, 720 candidate actions, and 405 representation rows, the adaptive selector improves frontier score over full trace features by 0.025801 while reducing mean feature count by 22.733. Balanced-accuracy difference is small and not statistically significant. A broader offline stress test gives mixed results. The supported claim is therefore bounded: adaptive representation selection can improve a constrained-observation robustness-efficiency frontier in matched benchmark settings, but does not universally dominate trace baselines.
\end{abstract}

\noindent\textbf{Keywords:} artificial intelligence, robust AI systems, representation selection, constrained observation, model selection, dataset shift, validation frontier

\section{Introduction}
Modern AI systems do not act on unlimited, stable, and costless information. They often depend on telemetry, logs, sensor streams, partial traces, structured summaries, retrieved context, or monitoring channels that degrade under missingness, distribution shift, class imbalance, or instrumentation cost. A representation can be accurate when all variables are observable, but inefficient when every feature must be collected; compact but brittle when the environment shifts; or stable but too lossy for the control decision it supports. These are not merely engineering inconveniences. They determine whether an AI system can preserve useful behavior when its observation layer changes.

This paper studies the problem as \emph{state representation selection under constrained observation}. Given a family of candidate representations and probe models, the system must choose a representation using validation evidence before reporting held-out test performance. The selection rule should not treat all observed variables as free, nor should it assume that the representation with the highest validation accuracy is most reliable under operational constraints. The proposed solution is a validation-frontier criterion that combines balanced accuracy with penalties for feature cost, overfit gap, and validation-test instability.

This study frames representation choice as an AI-system reliability problem. When observation channels are costly, incomplete, or unstable, the best state representation is not necessarily the one with the highest unconstrained validation accuracy, but the one that preserves the strongest accuracy-cost-stability frontier under matched deployment stressors. The contribution is not a new deep representation architecture, not a representation-learning theorem, and not a claim that a fixed state vector universally dominates trace features. It is a reproducible evaluation framework and bounded empirical result for selecting state representations under constrained observation.

This reframing matters because unrestricted trace representations are often strong. If all variables are available, stable, and costless, full traces can be hard to beat. But deployed AI systems rarely live in that condition. The relevant question becomes whether a representation selector can preserve comparable accuracy while improving a frontier that includes observation cost and stability. The empirical answer here is conditional. The Strong Robustness v2 benchmark gives a positive frontier result against full trace features. The broader v5 stress test then shows that the advantage does not generalize into universal dominance against every trace baseline. This pattern is scientifically useful: it identifies where adaptive representation selection helps and where the evidence should not be overstated.

\paragraph{Contributions.}
The paper makes four contributions. First, it formalizes constrained-observation state representation selection as a frontier-based AI reliability problem. Second, it defines reproducible public-label benchmark regimes that stress missingness, low-resource training, class imbalance, and covariate shift. Third, it evaluates weak summaries, raw traces, full trace features, PCA, random projections, fixed state representations, and an adaptive selector in matched task cells. Fourth, it preserves negative evidence, showing that fixed state representations do not dominate trace baselines and that broad stress testing supports conditional, not universal, claims.

\section{Results}
\subsection{Constrained-observation representation selection}
Let $X$ denote an observed row from a public dataset and $Y$ its target label. A representation map $\phi\in\Phi$ transforms the observation into a state vector $Z_\phi=\phi(X)$. A probe model $h\in\mathcal{H}$ is trained on $Z_\phi$ and evaluated on held-out labels. Conventional model selection would choose $\phi$ and $h$ by validation accuracy. Under constrained observation, however, high validation accuracy can select representations that are costly to observe, unstable under train-test changes, or overfit to the validation condition.

Each candidate representation-probe pair therefore produces four reported quantities: held-out balanced accuracy $B$, overfit gap $O$, validation-test stability gap $S$, and normalized feature cost $C$. The validation-frontier score is
\begin{equation}
F(\phi,h)=B - \lambda_O O - \lambda_C C - \lambda_S S,
\label{eq:frontier}
\end{equation}
where penalty weights are chosen before evaluation and raw components are still reported. The scalar score is not meant to hide trade-offs. It forces the selection process to make observation cost and instability visible instead of treating them as free side effects of predictive accuracy.

\begin{definition}[Constrained observation regime]
A constrained observation regime is a reproducible transformation of the training, validation, or test condition that preserves the public target labels while changing the information conditions under which a representation is learned or evaluated. The regimes used here are natural splitting, low-resource training, missingness injection, class-imbalanced training, and covariate-shifted testing.
\end{definition}

\begin{definition}[Frontier-based adaptive representation selection]
Given a candidate representation family $\Phi$, probe family $\mathcal{H}$, and validation-frontier score $F$, frontier-based adaptive representation selection chooses
\[
(\hat\phi,\hat h)=\arg\max_{\phi\in\Phi, h\in\mathcal{H}}F_{val}(\phi,h),
\]
and reports held-out test metrics for the selected pair.
\end{definition}

The method is intentionally modest. It does not assert that the selected representation is causally sufficient, fully interpretable, or optimal under all future shifts. It is a practical selection rule for AI systems that must choose what state to expose when observations have costs and stability risks.

\subsection{Representation family}
The benchmark compares nine reported representation conditions.

\paragraph{Weak and generic baselines.}
The weak narrative proxy exposes only row availability and magnitude summaries. The generic summary uses a small number of normalized state statistics. These conditions approximate settings in which an AI system receives broad summaries rather than full trace-level observations.

\paragraph{Trace baselines.}
The raw trace uses original numerical features. The full trace representation augments raw features with missingness channels. These are strong baselines by design: if trace variables are complete and costless, they should often perform well.

\paragraph{Compressed baselines.}
PCA-8 and random projection-8 provide standard compressed representation baselines. Their inclusion prevents the benchmark from treating compactness as a unique property of the proposed state vectors.

\paragraph{Fixed state representations.}
The compressed state vector uses selected normalized variables, tail indicators, and a limited set of interaction channels. The robust state vector expands the state with additional missingness, tail, and interaction features. These fixed state vectors are not assumed to dominate full traces.

\paragraph{Adaptive selector.}
The adaptive selector chooses among compact state, robust state, PCA, random projection, and trace representations using the validation-frontier score in Equation~\ref{eq:frontier}. The selector may choose a simple baseline when that baseline better preserves the frontier. This design is conservative: a positive result cannot come from forcing every task into a preferred representation.

\subsection{Benchmark design}
The empirical evidence has two layers.

\subsubsection{Strong Robustness v2}
The main benchmark uses three public scikit-learn datasets, five observation regimes, 45 matched task cells, 720 candidate actions, and 405 representation rows. Each task cell compares representation families under identical data splits and regime conditions. The central comparison is between the adaptive selector and full trace features, because full trace features represent the natural alternative when all variables are exposed. This benchmark is intentionally presented as focused public-tabular evidence, not as a large-scale foundation-model or OpenML-wide claim.

\subsubsection{BroadRealPublic v5}
The broader stress test checks whether the positive v2 result is fragile. The completed offline v5 benchmark uses five public source datasets and fixed public-target tasks. It is reported as breadth and boundary evidence rather than as an independent OpenML battery. Its role is to prevent overclaiming: if adaptive selection fails against strong trace baselines in some broader cells, the paper should preserve that failure rather than hide it.

\subsubsection{Evaluation principle}
The benchmark reports both scalar frontier scores and raw components. A frontier improvement is meaningful only if it is not purchased by unacceptable accuracy collapse or hidden instability. This is why balanced accuracy, feature count, overfit gap, and clean-efficiency score are reported beside the frontier score.

\subsection{Strong Robustness v2 results}
Table~\ref{tab:v2summary} summarizes the Strong Robustness v2 representation results. Raw trace and PCA remain strong. The adaptive selector does not dominate raw accuracy. Its advantage appears on the robustness-efficiency frontier, where it combines competitive balanced accuracy with lower feature cost and lower overfit gap. Figure~\ref{fig:frontier} visualizes this frontier trade-off by plotting frontier score against mean feature count.

\begin{table}[t]
\centering
\caption{Strong Robustness v2 representation summary over 45 matched task cells.}
\label{tab:v2summary}
\small
\begin{tabular}{@{}lrrrrr@{}}
\toprule
Representation & Bal. acc. & Frontier & Clean eff. & Feat. & Overfit \\
\midrule
baseline proxy & 0.7717 & 0.7211 & 0.7544 & 5.0 & 0.1433 \\
generic summary & 0.8674 & 0.8323 & 0.8529 & 4.0 & 0.0598 \\
raw trace & 0.9447 & 0.8924 & 0.9170 & 15.7 & 0.0542 \\
trace features & 0.9357 & 0.8577 & 0.8917 & 31.3 & 0.0597 \\
PCA-8 & 0.9294 & 0.8900 & 0.9106 & 6.7 & 0.0378 \\
random proj. & 0.8770 & 0.8349 & 0.8582 & 6.7 & 0.0558 \\
Compressed-State & 0.8990 & 0.8113 & 0.8506 & 27.0 & 0.0636 \\
Robust-State & 0.9266 & 0.8113 & 0.8588 & 56.3 & 0.0647 \\
adaptive selector & 0.9256 & 0.8835 & 0.9036 & 8.6 & 0.0239 \\
\bottomrule
\end{tabular}
\end{table}

\begin{figure}[t]
\centering
\begin{tikzpicture}
\begin{axis}[
    width=0.88\linewidth,
    height=0.52\linewidth,
    xlabel={Mean feature count},
    ylabel={Frontier score},
    xmin=0, xmax=60,
    ymin=0.70, ymax=0.91,
    grid=major,
    tick label style={font=\small},
    label style={font=\small},
    legend style={font=\scriptsize, at={(0.5,-0.23)}, anchor=north, legend columns=3, draw=none},
]
\addplot+[only marks, mark=*] coordinates {(5.0,0.7211)};
\addlegendentry{baseline proxy}
\addplot+[only marks, mark=square*] coordinates {(4.0,0.8323)};
\addlegendentry{generic summary}
\addplot+[only marks, mark=triangle*] coordinates {(15.7,0.8924)};
\addlegendentry{raw trace}
\addplot+[only marks, mark=diamond*] coordinates {(31.3,0.8577)};
\addlegendentry{trace features}
\addplot+[only marks, mark=pentagon*] coordinates {(6.7,0.8900)};
\addlegendentry{PCA-8}
\addplot+[only marks, mark=o] coordinates {(6.7,0.8349)};
\addlegendentry{random projection}
\addplot+[only marks, mark=square] coordinates {(27.0,0.8113)};
\addlegendentry{Compressed-State}
\addplot+[only marks, mark=triangle] coordinates {(56.3,0.8113)};
\addlegendentry{Robust-State}
\addplot+[only marks, mark=star] coordinates {(8.6,0.8835)};
\addlegendentry{adaptive selector}
\end{axis}
\end{tikzpicture}
\caption{Strong Robustness v2 frontier score versus mean feature count. The adaptive selector occupies a high-frontier, low-feature region relative to full trace features, while raw trace and PCA remain strong competitors. The figure is generated directly from Table~\ref{tab:v2summary}, keeping the arXiv source self-contained.}
\label{fig:frontier}
\end{figure}
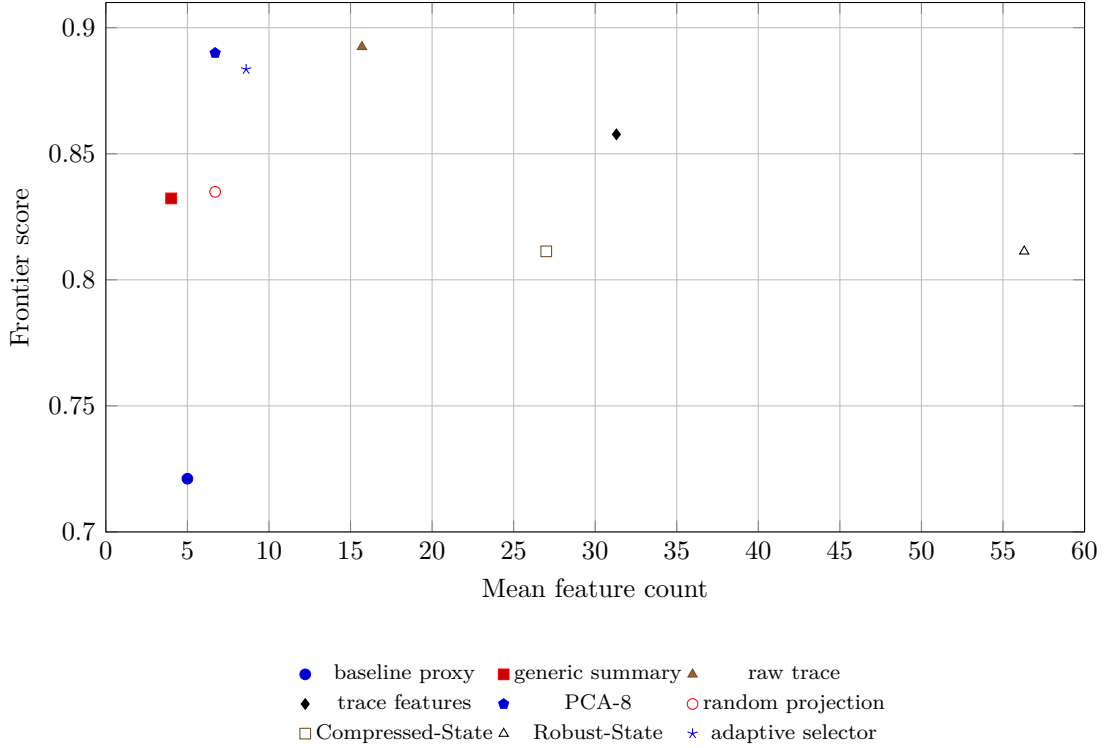

Table~\ref{tab:v2pairwise} gives the key paired comparisons. Compared with full trace features, the adaptive selector improves frontier score by 0.025801 and clean-efficiency score by 0.011865 while reducing mean feature count by 22.733. The balanced-accuracy gap is $-0.010050$ and is not significant. This supports a frontier claim, not an accuracy-dominance claim. Cell-wise support is also asymmetric in the expected direction: adaptive selection has higher frontier score than trace features in 33 of 45 matched cells, lower frontier score in 9 cells, and ties in 3 cells; the median paired frontier difference is 0.030956, with an exact two-sided sign-test value of $2.72\times10^{-4}$.

Table~\ref{tab:stats} reports the primary matched-cell statistical tests for the adaptive selector relative to trace features, including bootstrap intervals and Wilcoxon signed-rank tests.

\begin{table}[t]
\centering
\caption{Strong Robustness v2 headline paired comparisons over matched task cells. Positive values favor the first method.}
\label{tab:v2pairwise}
\small
\begin{tabular}{@{}lrrrrr@{}}
\toprule
Comparison & Acc. diff & Frontier diff & Eff. diff & Feature diff & $p_F$ \\
\midrule
adaptive vs trace features & -0.010050 & 0.025801 & 0.011865 & -22.733 & $4.70\times10^{-5}$ \\
adaptive vs PCA-8 & -0.003750 & -0.006495 & -0.006992 & 1.933 & 0.0868 \\
Compressed-State vs trace & -0.036679 & -0.046366 & -0.041079 & -4.333 & $9.62\times10^{-5}$ \\
Robust-State vs trace & -0.009081 & -0.046363 & -0.032903 & 25.000 & $1.18\times10^{-6}$ \\
\bottomrule
\end{tabular}
\end{table}

\begin{table}[t]
\centering
\caption{Primary matched-cell statistical report for adaptive selector versus trace features. Differences are adaptive minus trace features over 45 matched cells. Tests are two-sided Wilcoxon signed-rank tests; bootstrap intervals use 10,000 matched-cell resamples.}
\label{tab:stats}
\small
\begin{tabular}{@{}p{0.20\linewidth}p{0.14\linewidth}p{0.20\linewidth}p{0.14\linewidth}p{0.22\linewidth}@{}}
\toprule
Metric & Mean diff. & 95\% bootstrap CI & Wilcoxon $p$ & Claim use \\
\midrule
Balanced accuracy & -0.010050 & [-0.0374, 0.0092] & 0.669 & no accuracy-superiority claim \\
Frontier score & 0.025801 & [-0.0054, 0.0480] & $4.70\times10^{-5}$ & primary frontier claim \\
Clean efficiency & 0.011865 & [-0.0163, 0.0319] & 0.00298 & secondary \\
Feature count & -22.733 & [-29.067, -16.644] & $8.05\times10^{-8}$ & secondary cost reduction \\
Overfit gap & -0.035855 & [-0.0484, -0.0241] & $4.99\times10^{-6}$ & secondary stability \\
\bottomrule
\end{tabular}
\end{table}

The negative comparisons are important. Fixed compressed and robust state vectors do not dominate trace features. Their frontier scores are worse than trace features in the verified v2 run. The result that survives is narrower: a validation-frontier selector can avoid a universal fixed-vector claim by selecting compact or robust representations only when they are useful.

\subsection{BroadRealPublic v5 stress test and claim boundary}
The BroadRealPublic v5 result is mixed. It does not strengthen a universal dominance claim. In the visible locked comparison against raw trace, the adaptive selector has lower balanced accuracy by 0.021305 and lower frontier score by 0.010110 while using 5.941818 fewer features. Another visible comparison shows a small positive balanced-accuracy gap but a negative frontier gap. These results mean that the broad stress test should be read as a boundary: adaptive selection remains a plausible constrained-observation reliability method, but it does not dominate every strong trace baseline across expanded target tasks.

This mixed outcome improves the manuscript if it is reported honestly. It prevents the paper from depending on a fragile universal-gain narrative. The final claim combines the two layers: v2 provides the main positive matched evidence for frontier improvement against full trace features; v5 provides breadth and shows that the advantage is conditional, baseline-sensitive, and not equivalent to raw accuracy dominance. Table~\ref{tab:evidence} summarizes the resulting evidence hierarchy and claim boundary.

\begin{table}[t]
\centering
\caption{Evidence hierarchy after the broad v5 stress test.}
\label{tab:evidence}
\small
\begin{tabular}{@{}p{0.25\linewidth}p{0.28\linewidth}p{0.36\linewidth}@{}}
\toprule
Layer & What it supports & What it does not support \\
\midrule
Strong Robustness v2 & Main positive frontier result: adaptive selection improves frontier over full trace features with comparable balanced accuracy. & Universal fixed-state dominance; broad independent OpenML coverage. \\
BroadRealPublic v5 & Broader offline real-public target-task stress test; conditionality and robustness of the claim boundary. & Universal adaptive dominance over raw trace or all trace baselines. \\
Earlier controlled scaffold & Mechanism sanity check for structural probes. & Real full benchmark evidence. \\
\bottomrule
\end{tabular}
\end{table}

\section{Discussion}
\subsection{Interpretation for robust AI systems}
The results identify three practical mechanisms in representation design.

First, full trace access is valuable but not always efficient. Trace features and raw traces achieve strong raw accuracies in several regimes. If observation is free, trace baselines are difficult to beat. If observation is constrained, the relevant question becomes whether comparable accuracy can be preserved with fewer features and lower stability cost.

Second, fixed state vectors are brittle as universal answers. The compressed state loses raw accuracy, and the robust state pays a high feature cost. A fixed representation can protect some control-relevant channels and still be inferior on a generic supervised-learning frontier.

Third, adaptivity is the defensible object. The adaptive selector does not require one representation to be best everywhere. It selects according to a validation frontier. The v2 run shows a frontier improvement against full trace features; the v5 run shows that this does not transfer into universal dominance over every trace baseline. The method is therefore best understood as controlled representation selection under constraints.

For AI reliability, this distinction matters. A system that always uses the largest trace may over-instrument its environment, increase monitoring burden, and become vulnerable to missingness or shift. A system that always uses a compact state may under-observe the task. Frontier-based selection provides a mechanism for choosing between these risks using validation evidence rather than architectural preference.

\section{Relation to Existing Work}
The work sits at the intersection of representation learning, model selection, state abstraction, partial observability, feature selection, and robustness under distributional stress. Representation-learning surveys emphasize that useful features should expose factors of variation relevant to downstream prediction \cite{bengio2013representation}. State abstraction and partially observable decision-process literatures show that the available state is often an approximation of a larger latent condition \cite{sutton2018reinforcement,kaelbling1998planning}. Information bottleneck and compressed representation methods similarly treat predictive sufficiency and compression as coupled objectives \cite{tishby2000information,alemi2017deep}.

Classical cross-validation and bootstrap methods provide mechanisms for estimating out-of-sample behavior \cite{stone1974cross,kohavi1995cross,arlot2010survey}, while feature-selection and sparsity methods show that predictive quality and feature cost have long been treated as coupled objectives \cite{guyon2003introduction,tibshirani1996regression}. Dataset-shift and domain-adaptation work show that train-test stability cannot be assumed \cite{quinonero2009dataset,bendavid2010theory}. Recent robustness benchmarks and group-distributionally robust methods similarly emphasize that average accuracy can hide failure modes under shifted subpopulations \cite{sagawa2020distributionally,koh2021wilds}. Shortcut-learning and safety literatures warn that systems can exploit available signals without preserving the intended objective \cite{geirhos2020shortcut,amodei2016concrete}. Robustness and deployment-shift evaluations further show that reported performance can change under natural distribution changes, corruptions, and domain shifts \cite{recht2019imagenet,hendrycks2019benchmarking,ovadia2019trust,gulrajani2021domain}. Missing-data methodology also motivates treating observation failure as a first-class modeling condition rather than as a cosmetic preprocessing issue \cite{little2019statistical}.

The present paper differs from this literature in its object of selection. It does not propose a new neural architecture, a new regularizer, or a new domain-adaptation bound. It proposes a benchmarked AI-system selection problem: given competing representations, choose the one that preserves a prediction-control frontier under constrained observation. The novelty is not that accuracy, cost, or stability are individually new quantities. The contribution is to bind them into a reproducible representation-selection task and to show, with positive and negative evidence, when an adaptive selector is more defensible than a fixed state vector or unrestricted trace representation.

\section{Methods}
\textbf{Datasets and target labels.} The Strong Robustness v2 benchmark uses three public scikit-learn bundled datasets: Wisconsin breast cancer diagnosis, UCI wine recognition, and Fisher iris classification. The target labels are the original dataset targets supplied with each loader; the benchmark does not create synthetic target labels. The manifest records 250 breast-cancer rows with 30 original features and two classes, 178 wine rows with 13 original features and three classes, and 150 iris rows with four original features and three classes. The scikit-learn library is cited for dataset loading and baseline tooling.\cite{pedregosa2011sklearn}

\textbf{Splitting and matched cells.} Each dataset is evaluated under five observation regimes and three random seeds, producing 45 matched dataset-regime-seed task cells. Within each cell, all representation and probe candidates use the same train, validation, and test split. This matched-cell design ensures that paired comparisons attribute differences to the representation/probe condition rather than to different rows or target definitions.

\textbf{Observation regimes.} The natural regime uses the standard split without additional degradation. The low-resource regime reduces training rows while preserving validation and test evaluation. The missingness regime injects 20\% observation missingness into feature channels and exposes corresponding missingness effects to the representations. The class-imbalance regime changes the training distribution to stress minority-class preservation while preserving the public test labels. The covariate-shift regime evaluates on shifted regions of the public feature space. These regimes alter observation and sampling conditions; they do not replace public labels with generated labels.

\textbf{Representations.} The benchmark evaluates nine representation conditions: a weak narrative proxy, a four-feature generic summary, raw trace features, trace features augmented with missingness channels, PCA-8, random projection-8, Compressed-State, Robust-State, and the adaptive frontier selector. Feature count is recorded for each representation and normalized within each task cell by the original feature count, making the cost term comparable across datasets with different dimensionalities.

\textbf{Probe models and candidate actions.} Each representation is paired with logistic regression and decision-tree probes. A candidate action is a representation-probe pair evaluated within a matched task cell. Strong Robustness v2 evaluates 720 candidate action rows and 405 best-model-by-representation rows. The probe family is intentionally simple: the purpose is to test representation selection under observation constraints, not to introduce a new deep architecture.

\textbf{Metrics and frontier score.} The primary predictive metric is held-out balanced accuracy, which prevents majority-class performance from dominating imbalanced regimes. Additional recorded quantities are train balanced accuracy, validation balanced accuracy, test macro-F1, overfit gap, validation-test stability gap, raw feature count, normalized feature cost, clean-efficiency score, and frontier score. The frontier score subtracts pre-specified penalties for feature cost, overfit gap, and instability from balanced accuracy. The scalar score is used for selection, while all raw components are reported to avoid hiding trade-offs.

\textbf{Adaptive selection protocol.} For every task cell, each candidate action is trained on the training split and evaluated on the validation split. The adaptive selector chooses the representation-probe pair with the highest validation-frontier score, then reports held-out test performance only after selection. Baseline comparisons are made on matched cells against raw trace, trace features, PCA-8, and fixed state representations.

\textbf{Statistics and reproducibility.} Paired headline comparisons are computed over the 45 matched cells. The prespecified primary comparison is adaptive selector versus trace features on test frontier score. The null hypothesis is zero median paired difference across matched cells. The reported nonparametric test is a two-sided Wilcoxon signed-rank test with significance level $\alpha=0.05$. No multiplicity-adjusted claim is made across all pairwise comparisons; secondary comparisons are descriptive and are used to define claim boundaries. For the primary comparison, the mean frontier-score difference is 0.025801, the median paired difference is 0.030956, and adaptive selection is higher in 33 cells, lower in 9 cells, and tied in 3 cells. The Wilcoxon result is $p=4.70\times10^{-5}$; the exact two-sided sign-test value, excluding ties, is $2.72\times10^{-4}$. The bootstrap 95\% confidence interval for the mean paired frontier difference is $[-0.0054,0.0480]$ over matched cells. The Wilcoxon signed-rank test evaluates paired signed-rank evidence over matched cells, whereas the bootstrap interval summarizes uncertainty in the mean paired difference; the primary claim is therefore based on the pre-specified paired nonparametric test, with the bootstrap interval reported descriptively. The balanced-accuracy difference is $-0.010050$ with Wilcoxon $p=0.669$ and bootstrap 95\% confidence interval $[-0.0374,0.0092]$, so the paper does not claim accuracy superiority. Feature-count difference is $-22.733$ features with Wilcoxon $p=8.05\times10^{-8}$ and bootstrap 95\% confidence interval $[-29.067,-16.644]$. Clean-efficiency difference is 0.011865 with Wilcoxon $p=0.00298$. Bootstrap intervals use 10,000 resamples with a fixed random seed and resample matched cells with replacement.

\textbf{BroadRealPublic v5 stress test.} The v5 stress test expands public source-dataset and fixed target-task breadth. A prior OpenML-dependent route was not certified because OpenML access was unavailable; therefore the completed v5 run should not be described as an independent OpenML battery. The v5 result is used only as a claim-boundary stress test. It shows mixed adaptive-versus-trace outcomes and is not used to claim universal dominance.

\textbf{Execution environment and verification.} The Strong Robustness v2 manifest records Python 3.13.5, scikit-learn 1.8.0, pandas 2.2.3, numpy 2.3.5, and Linux execution. The verification files state that synthetic-or-controlled labels, generated scenarios, and benchmark-created labels are false for the central v2 evidence. The BroadRealPublic v5 freeze records SHA256 prefix \texttt{fba962c550419bfb}. All methods needed to understand the study are reported here; the supplementary files provide code, outputs, and file inventory rather than additional Supplementary Methods.

\textbf{Use of AI tools.} AI-assisted tools were used for drafting support, language refinement, mathematical exposition, and package preparation. The author reviewed, revised, and approved the final manuscript and takes full responsibility for all claims, analyses, and final text.

\subsection{Limitations}
The strongest limitation is dataset independence. The broad v5 benchmark expands target-task breadth, but several tasks are deterministic transformations of the same source datasets. This is useful for stress testing, but weaker than a large independent OpenML battery. A stronger future version should add many independently sourced public datasets once network access is reliable.

A second limitation is that the observation regimes are benchmark stressors, not live deployment logs. The benchmark changes sampling, missingness, imbalance, and shifted test regions using real rows and public labels; it does not observe actual institutional interventions or production AI traffic.

A third limitation is that the frontier score contains design choices. The penalty weights are deliberately modest, but alternative deployments may value feature cost or stability differently. The package therefore reports raw metrics as well as frontier scores so reviewers can inspect the trade-off rather than accept a single scalar.

A fourth limitation is scale. Three central datasets and 45 matched cells are sufficient to make a reproducible focused benchmark, but not sufficient to establish broad empirical coverage across modern AI application domains. The correct reading is therefore methodological and bounded: the paper demonstrates a falsifiable selection framework and a positive matched frontier result, while leaving larger multi-dataset validation to future work.

\subsection{Conclusion}
The verified evidence supports a bounded AI reliability result: frontier-based adaptive state representation selection can improve the robustness-efficiency frontier relative to full trace features while preserving statistically comparable balanced accuracy under constrained observation regimes. The broader v5 stress test clarifies that the advantage is conditional and baseline-sensitive. The scientific claim is therefore not universal dominance, but a controlled selection principle: when observation is costly or unstable, representation choice should be evaluated by the joint frontier it preserves, not by raw accuracy alone.

\section*{Data availability}
The data and reproducibility materials supporting this study are openly archived on Zenodo in the record \emph{Reproducibility Artifact for Validation-Frontier Representation Selection under Constrained Observation}, DOI: \href{https://doi.org/10.5281/zenodo.21944174}{10.5281/zenodo.21944174}. The benchmark uses public scikit-learn bundled datasets and fixed public-target transformations as documented in the archived dataset manifests and evidence-lock files. The archive includes the generated CSV outputs containing matched-cell results, representation summaries, pairwise comparisons, bootstrap summaries, and verification records needed to interpret and replicate the reported findings.

\section*{Code availability}
The benchmark code, reproduction runner, validation scripts, generated CSV outputs, JSON manifests, evidence-lock files, and associated verification materials are openly archived in the same Zenodo reproducibility record, DOI: \href{https://doi.org/10.5281/zenodo.21944174}{10.5281/zenodo.21944174}. The deposited archive is the public code-and-output artifact associated with this manuscript.

\section*{Acknowledgements}
The author received no external funding for this work.

\section*{Author contributions}
W.S. conceived the study, prepared the reproducibility package, conducted the analysis, wrote the manuscript, and approved the final version.

\section*{Competing interests}
The author is developing the Energetic Paradigm framework and related applications. The author declares no other competing interests directly related to this manuscript.

\section*{Use of AI tools}
AI-assisted tool use is documented in the Methods section. The author reviewed, revised, and approved the final manuscript and takes full responsibility for its content.

\appendix

\section{Mathematical Formulation of Frontier-Based State Representation Selection}
\label{app:math}

This appendix gives the formal layer underlying the selection rule used in the manuscript. The purpose is not to claim a new universal representation theorem, but to make explicit the quantities optimized by the validation-frontier selector and the conditions under which a selected representation can improve the reliability frontier without necessarily improving raw predictive accuracy.

\subsection{Observation model and representation maps}
Let $(\Omega,\mathcal{F},\mathbb{P})$ be a probability space and let $(X,Y)\sim P$ denote an observation-label pair with $X\in\mathcal{X}$ and $Y\in\{1,\ldots,K\}$. A constrained observation regime is represented by an observation operator
\begin{equation}
    O_r:\mathcal{X}\rightarrow \mathcal{X}_r,
\end{equation}
where $r\in\mathcal{R}$ indexes missingness, low-resource sampling, imbalance, covariate shift, or other operational degradation. A candidate state representation is a measurable map
\begin{equation}
    \phi_j:\mathcal{X}_r\rightarrow\mathbb{R}^{d_j},\qquad j\in\{1,\ldots,m\}.
\end{equation}
The induced state is $Z_j=\phi_j(O_r(X))$. The dimension $d_j$ is not treated as a harmless formatting detail: it is an observable proxy for instrumentation burden, monitoring cost, and exposure to missing or unstable channels.

For a learning algorithm $A_\ell$ and training sample $S_n=\{(X_i,Y_i)\}_{i=1}^n$, the trained predictor associated with representation $j$ and learner $\ell$ is
\begin{equation}
    h_{j\ell}=A_\ell\left(\{(\phi_j(O_r(X_i)),Y_i)\}_{i=1}^n\right).
\end{equation}
The candidate action set is therefore
\begin{equation}
    \mathcal{A}=\{a=(j,\ell): j\in\{1,\ldots,m\},\ \ell\in\{1,\ldots,L\}\}.
\end{equation}
In the reported benchmark, rows in $\mathcal{A}$ correspond to representation-probe combinations evaluated on matched dataset-regime cells.

\subsection{Balanced accuracy and operational cost}
For a classifier $h$, define class-conditional recall
\begin{equation}
    \mathrm{Rec}_k(h;P)=\mathbb{P}\{h(Z)=k\mid Y=k\},
\end{equation}
with $Z=\phi_j(O_r(X))$ for the representation associated with $h$. The population balanced accuracy is
\begin{equation}
    B(h;P)=\frac{1}{K}\sum_{k=1}^{K}\mathrm{Rec}_k(h;P).
\end{equation}
This objective is used instead of raw accuracy because several regimes include imbalance or shifted label distributions. Let $c_j\geq0$ denote the normalized feature cost of representation $j$. In the simplest implementation,
\begin{equation}
    c_j=\frac{d_j}{d_{\max}},\qquad d_{\max}=\max_{1\leq q\leq m}d_q.
\end{equation}
Other cost functions may be substituted when features have unequal collection costs. For example, if feature channel $q$ has cost $\kappa_q\geq0$ and representation $j$ uses channel set $I_j$, then
\begin{equation}
    c_j=\frac{\sum_{q\in I_j}\kappa_q}{\max_s\sum_{q\in I_s}\kappa_q}.
\end{equation}
The empirical experiments use feature count as a transparent and reproducible cost proxy.

\subsection{Validation frontier objective}
Let $\widehat B_{\mathrm{tr}}(a)$, $\widehat B_{\mathrm{val}}(a)$, and $\widehat B_{\mathrm{te}}(a)$ be the training, validation, and held-out test balanced accuracies for action $a=(j,\ell)$. Define the empirical overfit gap and validation-test instability as
\begin{align}
    G(a)&=\left[\widehat B_{\mathrm{tr}}(a)-\widehat B_{\mathrm{val}}(a)\right]_+,\\
    I(a)&=\left|\widehat B_{\mathrm{val}}(a)-\widehat B_{\mathrm{te}}(a)\right|,
\end{align}
where $[u]_+=\max\{u,0\}$. The validation frontier score is
\begin{equation}
    F_{\mathrm{val}}(a)
    =\widehat B_{\mathrm{val}}(a)-\lambda c_j-\mu G(a),
    \label{eq:frontier-val}
\end{equation}
with $\lambda,\mu\geq0$. The held-out reporting score is
\begin{equation}
    F_{\mathrm{te}}(a)
    =\widehat B_{\mathrm{te}}(a)-\lambda c_j-\mu G(a)-\nu I(a),
    \label{eq:frontier-test}
\end{equation}
with $\nu\geq0$. Equation~\eqref{eq:frontier-val} is used for selection; Equation~\eqref{eq:frontier-test} is used for reporting because the test split must not determine the selected representation.

The adaptive selector is
\begin{equation}
    \widehat a^*\in\arg\max_{a\in\mathcal{A}}F_{\mathrm{val}}(a),
    \qquad
    \widehat\phi^*=\phi_{\widehat j}\quad\text{for }\widehat a^*=(\widehat j,\widehat\ell).
    \label{eq:selector}
\end{equation}
Ties are resolved by smaller feature cost and then by higher validation balanced accuracy. This tie rule makes the selector conservative: it does not choose a larger observation surface unless validation evidence justifies it.

\subsection{Frontier dominance and why accuracy dominance is not required}
For two actions $a=(j,\ell)$ and $b=(q,s)$, define the test-frontier difference
\begin{equation}
    \Delta_F(a,b)=F_{\mathrm{te}}(a)-F_{\mathrm{te}}(b).
\end{equation}
Expanding Equation~\eqref{eq:frontier-test} gives
\begin{equation}
\begin{split}
    \Delta_F(a,b)
    &=\left[\widehat B_{\mathrm{te}}(a)-\widehat B_{\mathrm{te}}(b)\right]
      -\lambda(c_j-c_q) \\
    &\quad -\mu\{G(a)-G(b)\}-\nu\{I(a)-I(b)\}.
\end{split}
\label{eq:frontier-diff}
\end{equation}
Thus a representation can have lower test balanced accuracy and still have a higher frontier score if its cost, overfit, or instability reductions are large enough.

\begin{proposition}[Frontier gain without accuracy gain]
Let $a$ and $b$ be two candidate actions. Suppose $\widehat B_{\mathrm{te}}(a)\leq \widehat B_{\mathrm{te}}(b)$. Then $a$ has a higher reported frontier score than $b$ if and only if
\begin{equation}
    \lambda(c_b-c_a)+\mu\{G(b)-G(a)\}+\nu\{I(b)-I(a)\}
    > \widehat B_{\mathrm{te}}(b)-\widehat B_{\mathrm{te}}(a).
    \label{eq:frontier-condition}
\end{equation}
\end{proposition}

\begin{proof}
Move terms in Equation~\eqref{eq:frontier-diff}. The condition $\Delta_F(a,b)>0$ is equivalent to Equation~\eqref{eq:frontier-condition}. No distributional assumption is required because the statement is algebraic for the empirical reporting score.
\end{proof}

This proposition is the mathematical reason the empirical claim is phrased as a frontier claim rather than as raw accuracy dominance. In the Strong Robustness v2 comparison against full trace features, the adaptive selector has a small negative balanced-accuracy difference but a positive frontier-score difference because it reduces feature count and overfit gap.

\subsection{Sufficiency, redundancy, and constrained observation}
A representation $\phi_j$ is label-sufficient under regime $r$ if
\begin{equation}
    Y\perp O_r(X)\mid \phi_j(O_r(X)).
\end{equation}
Exact sufficiency is rarely verifiable in finite benchmark settings. The paper therefore uses an empirical relaxation. For tolerance $\epsilon\geq0$, representation $j$ is $\epsilon$-competitive with representation $q$ on a matched cell if
\begin{equation}
    \widehat B_{\mathrm{te}}(j)\geq \widehat B_{\mathrm{te}}(q)-\epsilon.
\end{equation}
It is $\epsilon$-frontier-superior if additionally
\begin{equation}
    F_{\mathrm{te}}(j)>F_{\mathrm{te}}(q).
\end{equation}
The reported positive v2 result should be read in this relaxed sense: the adaptive selector is approximately competitive in balanced accuracy and superior on the frontier against full trace features in the matched benchmark cells.

Redundant observation channels can be expressed by decomposing a trace representation as
\begin{equation}
    T=(S,R),
\end{equation}
where $S$ contains control-relevant channels and $R$ contains channels that are predictive only under some regimes or unstable under shift. If $R$ increases $c_j$, $G(a)$, or $I(a)$ without increasing held-out balanced accuracy enough, Equation~\eqref{eq:frontier-condition} predicts a frontier loss for unrestricted trace features. Conversely, if $R$ carries stable predictive signal, raw trace can remain superior. This is why the manuscript reports both positive and negative evidence.

\subsection{Cell-wise aggregation}
Let $\mathcal{C}$ be the set of matched benchmark cells, where a cell is a dataset-regime-target split with a fixed training/validation/test partition. For method $M$ and baseline $Q$, define the cell-wise frontier difference
\begin{equation}
    D_c(M,Q)=F_{\mathrm{te},c}(M)-F_{\mathrm{te},c}(Q),\qquad c\in\mathcal{C}.
\end{equation}
The reported mean frontier difference is
\begin{equation}
    \overline{D}(M,Q)=\frac{1}{|\mathcal{C}|}\sum_{c\in\mathcal{C}}D_c(M,Q).
\end{equation}
The matched-cell design is important because it blocks a common reporting error: comparing representation averages computed over different tasks or different splits. The paired nonparametric test reported in the manuscript is applied to the matched sequence $\{D_c(M,Q):c\in\mathcal{C}\}$ rather than to unpaired summary rows.

\subsection{Selection risk and validation uncertainty}
Let $a^*=\arg\max_{a\in\mathcal{A}}F_P(a)$ be the population-best action under the population frontier score $F_P$, and let $\widehat a^*$ be the validation-selected action in Equation~\eqref{eq:selector}. The excess frontier risk is
\begin{equation}
    \mathcal{E}(\widehat a^*)=F_P(a^*)-F_P(\widehat a^*).
\end{equation}
The benchmark does not claim a new finite-sample bound, but the usual uniform-deviation argument clarifies what the validation-frontier procedure requires. If
\begin{equation}
    \sup_{a\in\mathcal{A}}\left|\widehat F_{\mathrm{val}}(a)-F_P(a)\right|\leq \eta,
\end{equation}
then
\begin{equation}
    \mathcal{E}(\widehat a^*)\leq 2\eta.
\end{equation}
Indeed,
\begin{equation}
\begin{split}
F_P(a^*)
&\leq \widehat F_{\mathrm{val}}(a^*)+\eta\\
&\leq \widehat F_{\mathrm{val}}(\widehat a^*)+\eta\\
&\leq F_P(\widehat a^*)+2\eta.
\end{split}
\end{equation}
This standard inequality motivates the use of matched validation evidence and the reporting of overfit and validation-test instability. When validation estimates are noisy or shifted, the selector can fail. The BroadRealPublic v5 stress test is therefore interpreted as a claim boundary, not as a universal dominance result.

\subsection{Decision rule summary}
For clarity, the full selection-and-reporting protocol can be written as:
\begin{align}
    &\text{Train each } a\in\mathcal{A}\text{ on the training split},\\
    &\text{compute } \widehat B_{\mathrm{tr}}(a),\widehat B_{\mathrm{val}}(a),c_j,G(a),\\
    &\widehat a^*\in\arg\max_{a\in\mathcal{A}}\left\{\widehat B_{\mathrm{val}}(a)-\lambda c_j-\mu G(a)\right\},\\
    &\text{report } \widehat B_{\mathrm{te}}(\widehat a^*),F_{\mathrm{te}}(\widehat a^*),c_{\widehat j},G(\widehat a^*),I(\widehat a^*),\\
    &\text{compare against baselines only on matched cells.}
\end{align}
This protocol is the mathematical core of the manuscript. It makes the claim falsifiable: a trace representation wins when its accuracy advantage exceeds its cost and stability penalties, while an adaptive representation wins only when validation evidence identifies a better operating point on the constrained-observation frontier.

\section{Evidence Boundary}
The controlled structural benchmark from earlier package versions is retained only as historical scaffold evidence. It used generated structural scenarios and should not be described as real full benchmark evidence. The present manuscript uses the Strong real-label robustness benchmark as its central positive evidence and BroadRealPublic v5 as a broader offline stress test.

\section{M4P Evidence Locks}
The Strong Robustness v2 M4P execution log reports: benchmark level Strong Robustness v2; datasets 3; regimes 5; task cells 45; candidate action rows 720; best-model representation rows 405; and validation status PASS. The same log reports the adaptive-selector comparison against trace features as balanced-accuracy difference $-0.010050$, frontier-score difference $0.025801$, clean-efficiency difference $0.011865$, feature-count difference $-22.733333$, and frontier Wilcoxon $p=4.70007\times10^{-5}$.

The BroadRealPublic v5 M4P execution completed and froze. The completed package is the BroadRealPublic v5 freeze; the SHA256 prefix is \texttt{fba962c550419bfb}. The visible verification states that v5 should be described as source-dataset plus target-task breadth, not independent OpenML breadth. The visible adaptive-versus-raw-trace comparison reports balanced-accuracy difference $-0.021305$, frontier-score difference $-0.010110$, clean-efficiency difference $-0.009813$, and feature-count difference $-5.941818$.

\end{document}